\documentclass[journal,letterpaper]{IEEEtran}

\usepackage{graphics} 
\usepackage{graphicx}

\usepackage{epsfig} 
\usepackage{amsmath} 
\usepackage{amssymb}  

\usepackage{amsfonts,amsmath,amssymb,pgfplots,siunitx,amsthm,cite}
\usepackage{float}
\usepackage{footnote}
\usepackage{subfig}

\usepackage{color}
\usepackage{comment}
\newtheorem{thm}{Theorem}
\newtheorem{prob}{Problem}

\newtheorem{asm}{Assumption}
\newtheorem{rem}{Remark}

\title{Biologically Inspired Collision Avoidance Without Distance Information}

\author{Thiago~Marinho $^{1}$,
        Massi~Amrouche $^{2}$,
        Du{\v s}an~Stipanovi\'{c} $^{2}$,
        Venanzio Cichella $^{3}$,
        and~Naira~Hovakimyan $^{1}$
\thanks{$^{1}$ Thiago Marinho, and Naira Hovakimyan are with the Mechanical Engineering Department and the Coordinated Science Laboratory, University of Illinois at Urbana-Champaign, Urbana, IL 61801 USA, email:{ \tt \small \{marinho, nhovakim\}@illinois.edu}}
\thanks{$^{2}$  Massi Amrouche and Du{\v s}an Stipanovi\'{c} are with the Department of Industrial and Enterprise Systems Engineering  and the Coordinated Science Laboratory, University of Illinois at Urbana-Champaign, Urbana, IL 61801 USA email:{ \tt \small \{amrouch2,dusan\}@illinois.edu} }
\thanks{$^{3}$  Venanzio Cichella is with the Department of Mechanical Engineering, The University of Iowa, Iowa City, IA 52242 USA email:{ \tt \small venanzio-cichella@uiowa.edu} }}
\begin{document}
\maketitle


\begin{abstract}

Biological evidence shows that animals are capable of evading eminent collision without using depth information, relying solely on \textit{looming} stimuli. In robotics, collision avoidance among uncooperative vehicles requires measurement of relative distance to the obstacle.  Small, low-cost mobile robots and UAVs might be unable to carry distance measuring sensors, like LIDARS and depth cameras.  We propose a control framework suitable for a unicycle-like vehicle moving in a 2D plane that achieves collision avoidance. The control strategy is inspired by the reaction of invertebrates to approaching obstacles, relying exclusively on line-of-sight (LOS) angle, LOS angle rate, and time-to-collision as feedback. Those quantities can readily be estimated from a monocular camera vision system onboard a mobile robot. The proposed avoidance law commands the heading angle to circumvent a moving obstacle with unknown position, while the velocity controller is left as a degree of freedom to accomplish other mission objectives. Theoretical guarantees are provided to show that minimum separation between the vehicle and the obstacle is attained regardless of the exogenous tracking controller.
\end{abstract}

\section{Introduction}
Autonomous vehicles, both ground and aerial, are quickly gaining popularity providing many benefits to society. Impacts are already being seen in urban mobility, aerial inspection, precision agriculture, surveillance, and healthcare.  As an example, an autonomous drone is envisioned to drastically improve the delivery bandwidth in the last mile problem scenario, which is an active  field of research. When these vehicles are small enough, they can be designed to fly indoors to help individuals with limited mobility such as in elderly care or in-home rehabilitation \cite{marinho2016carebots}.

To achieve a high level of autonomy, vehicles are expected to carry a reliable onboard collision avoidance system (CAS).  The safety of autonomous robots depends on the CAS's ability to deal with unpredicted events. As an example, a self-driving car cruising on a highway must be prepared to avert an imminent collision with a cyclist that suddenly comes onto a collision course. Reactive components are of paramount importance for safety towards addressing the challenges with imminent and unplanned spatial deconfliction needs.

When identifying obstacles and collision threats, autonomous vehicles and larger drones rely on expensive and sometimes cumbersome sensing equipment \cite{spriesterbach2013unmanned} and  \cite{wolcott2014visual}. Self-driving cars can be equipped with RADARs, LIDARs, and multiple cameras. Skydio's self-flying camera drone relies on more than ten cameras to construct a voxel map of the world around it. Smaller vehicles that are limited to lightweight and affordable solutions, usually like monocular cameras, are unable to carry sensors that can measure the distance to a moving obstacle. This restriction directly limits the obstacle avoidance capabilities, which are largely dependent on the amount of information made available by the sensors.

The performance commonly associated with the safety of successful autonomous systems relies on the measurement or direct estimation of the distance to the moving obstacles. 
However, it is well known that low-cost monocular cameras are not capable of estimating the relative position to a moving obstacle~\cite{alenya2009}. Many efforts have been made to acquire relative distance using RGB-D  cameras~\cite{Bachrach2012-bg} and stereo cameras~\cite{Hrabar2008-kz,Barry2015-xs}, but these technologies only provide reliable measurement in short range.  For unknown environments, simultaneous localization and mapping (SLAM) uses scanning with onboard LIDARs to identify the location of obstacles indoors~\cite{KumarICRA011}. Additionally, \cite{achtelik2011} shows that affordable cameras are suitable for navigation in unknown spaces, although these methods only perform environment mapping, not obstacle tracking. 

The collision avoidance problem in the absence of position information of the obstacle is underrepresented in the literature. Examples are \cite{voos2007uav,degen2011reactive,sharma2012reactive} which show that a collision can be avoided using image-based features (such as image area expansion, relative bearing rate), from which it is possible to estimate the range to the UAV from obstacle \cite{voos2007uav}, or the time-to-collision \cite{degen2011reactive}.  However, these solutions validate the collision avoidance algorithms through experimental results only.

For a class of smaller vehicles, it is  important to investigate a collision deconfliction solution that does not rely on distance measurement.  In this context, this paper proposes a collision avoidance solution when range information cannot be measured or estimated through the available sensors. We propose an output feedback control framework inspired by biology of visual guidance \cite{gibson2014ecological}, where stimuli such as \textit{loom} and \textit{time-to-collision} are used to drive the avoidance behavior. Additionally, under realistic assumptions, we provide safety guarantees for this autonomous navigation framework.

This work follows from recent efforts to provide theoretical results for collision avoidance without distance measurement \cite{cichella2015collision} and \cite{marinho2018guaranteed}.

\section{Biological Inspiration}
Most animals can achieve their most basic tasks like navigate in an environment, escape from a potential predator, or avoid an imminent collision, using only visual information as feedback. Although, wide range of species have access to the depth perception provided by the binocular information, it is fascinating to note that most of invertebrate motions and evasion maneuvers are solely based on  monocular stimulus \cite{dill1974escape, gibson2014ecological}.

The expansion of the retinal image is the critical component of deciding if an object is on a direct collision course with an observer \cite{gibson2014ecological}.  The size of the object on the retina is encoded in the angle $\theta$ defined in the simplified eye model in Figure \ref{fig:eye_model}.  In some sense, $\theta$ represents the object's size, while $\dot \theta$ - the expansion rate of the object on the retina. Previous studies \cite{regan1993dissociation} have demonstrated empirically that the quantity that humans particularly react to is neither  $\theta$ nor  $\dot \theta$ alone, but to the ratio $\frac{\theta}{\dot \theta}$. This quantity provides the basis for collision avoidance when driving on the motorway and during ball hitting \cite{regan1993dissociation}.

\begin{figure}[h!]
\centering
\includegraphics[width = 0.55\linewidth]{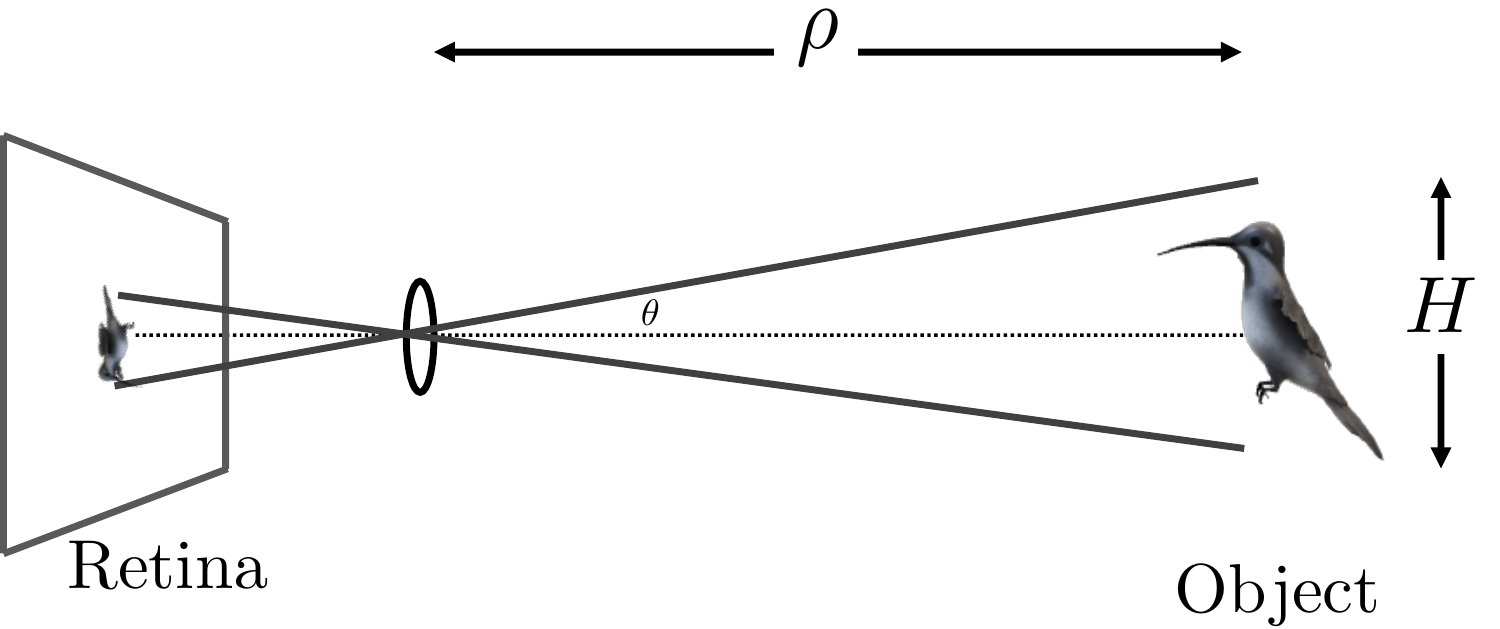}
\caption{Geometry of an object projected on the retina.}
\label{fig:eye_model}
\end{figure}

It can be shown that for smaller values of $\theta$, the relationship $\frac{\theta}{\dot \theta}$ is an approximation of $\frac{\rho}{\dot \rho}$ which is the definition of time-to-collision. To understand how time-to-collision can be measured from monocular cameras the reader is directed towards previous work from the authors \cite{marinho2018guaranteed}. Inspired by the work in \cite{moshtagh2009vision}, the analysis and control carried in this paper will be expressed as the loom $l: \mathbb{R} \times  \mathbb{R} \rightarrow \mathbb{R}_0^{-}$:

\begin{equation}
\label{eq:loom}
l\left(\rho(t), \dot \rho (t)\right) = \min\left\{\frac{\dot\rho(t)}{\rho(t)}, 0\right\}.
\end{equation}

\begin{rem}
Loom is the negative inverse of the time-to-collision $\tau$, and is the mathematical quantification of the looming stimuli observed in animals  \cite{joarder1992autonomous}. Larger absolute values of loom directly imply danger, and small absolute values of loom suggest the obstacle is not an imminent threat.
\end{rem}

\section{Problem Formulation}

The problem setting takes place in the $\mathbb{R}^2$-Euclidean space. Two-dimensional collision avoidance is the chosen formulation because it is employable in a three-dimensional navigation problem, as long as the avoidance maneuvers are performed in any chosen plane.

The central assumption to this body of work is that monocular camera-based vision systems are not suitable to accurately measure the distance and relative velocity with respect  to objects. Therefore, we consider that the measurements of distance and relative speed to the obstacle are unavailable. Our proposed strategy must guarantee collision avoidance in such a framework.


Define the position of the evading vehicle at time $t$ as $\vec{p}_\text{r}(t) = [x_\text{r}(t) \, , \, y_\text{r}(t)]^\top$.
Then, let the motion of the vehicle be driven by the unicycle like model:
\begin{equation} \label{eq:dynUAV}
\begin{cases}
\dot x_\mathrm{r}(t) = V_\mathrm{r}(t) \cos \psi_\mathrm{r}(t) \, , \quad & x_\mathrm{r,0} = x_\mathrm{r}(0) \\
\dot y_\mathrm{r}(t) = V_\mathrm{r}(t) \sin \psi_\mathrm{r} (t)\, , \quad & y_\mathrm{r,0} = y_\mathrm{r}(0)
\end{cases}, \quad
\end{equation}
where $V_\mathrm{r}(t)$ and $\psi_\mathrm{r} (t)$ are the speed and heading angle of the vehicle, respectively, and $u(t) = \dot \psi_\mathrm{r}(t)$ is the controlled angular rate. The proposed algorithm is suited for any unicycle model, which represents a large class of vehicles, such as constant altitude fixed wing UAVs and differential drive ground robots. The collision avoidance component of the control is added to the heading tracking controller angle as introduced in~\cite{cichella2015collision, marinho2018guaranteed}. The structure of the heading control law is defined as:

\begin{equation}
\label{eq:decouple}
   \dot \psi_{\mathrm{r}}(t) = \dot\psi_{\mathrm{tr}}(t) + \dot\psi_{\rm{ca}}(t),
\end{equation}
 where $u_\mathrm{ca}(t) = \dot\psi_{\rm{ca}}(t)$ is the avoidance component.
\begin{rem}[Decoupling of heading angle and speed]
The avoidance is performed solely by modifying the heading angle rate. We can assume that  $V_\mathrm{r}(t)$ and $u_\mathrm{tr} = \dot \psi_{\rm{tr}}(t)$ are given by an exogenous control law responsible for the robot's mission, such as a trajectory tracking, flocking/formation control, way point navigation or to satisfy temporal constraints.
\end{rem}

%
%
The geometry of the considered collision avoidance problem is shown in Figure \ref{fig:geom}. The heading angle $\psi_\mathrm{r}$ of the robot is defined with respect to the horizontal component of an inertial frame. Similarly, $\lambda$ is the angle of the line-of-sight with respect to the inertial frame. The distance between the robot and the obstacle is defined as $\rho$.
%
Furthermore, let $\vec{V}_\mathrm{r}$ and $\vec{V}_\mathrm{o}$ denote the robot's and obstacle's velocity on the avoidance plane. Before providing a formal statement of the problem at hand, we formulate a set of assumptions that the vehicle and obstacle must satisfy.

Similar to the evading vehicle, the obstacle's dynamics are  given by a unicycle model $
\dot x_\text{o}(t) = V_\text{o}(t) \cos \psi_\text{o}(t)$,
$\dot y_\text{o}(t) = V_\text{o}(t) \sin \psi_\text{o}(t),$ where $V_\mathrm{o}(t)$ and $\psi_\mathrm{o}(t)$ are the unknown speed and heading angle of the obstacle.
\begin{asm}
\label{asm:vel}
The vehicle and the obstacle velocities are bounded and there exist known constants $V_\mathrm{r, max} > V_\mathrm{r, min} > 0$, $V_\mathrm{o, max}>0$ and $\dot \psi_\mathrm{o, max}>0$, such that
\begin{align}
V_\mathrm{r}(t) \in  [V_\mathrm{r, min}, \; V_\mathrm{r, max}], \quad \forall t \geq 0,
\label{eq:boundvehiclespeed}
\end{align}
and
\begin{align}
0 &\leq V_\mathrm{o} (t)\leq V_\mathrm{o, max} \quad \text{ and } \quad \vert \dot \psi_\mathrm{o}(t)\vert \leq  \dot \psi_\mathrm{o, max}.
\label{eq:boundobstacle}  \end{align}

\end{asm}

The available measurements are line-of-sight angle (also called bearing angle), line-of-sight angle rate (optical flow) and time-to-collision. This sensing formulation has previously been used for feedback  in vision based formation control \cite{moshtagh2009vision}. The control objective is twofold: (i) avoid an unknown uncooperative obstacle and (ii) maintain a minimum time-to-collision.  





Real-time measurement of loom with a monocular camera is equivalent to measuring time-to-collision. Hence, we can state the following assumptions based on the problem formulation in \cite{marinho2018guaranteed} and \cite{moshtagh2009vision}.
\begin{asm}
	\label{asm:loom}
	$l(t)$ is an available measurement, although $\rho(t)$ and $\dot \rho(t)$ are unknown.
\end{asm}

\begin{asm}
\label{asm:sens_available}
The variables $l(t)$, $\lambda(t)$ and $\dot \lambda(t)$ are the only exogenous quantities available to the collision avoidance system, by means of the line-of-sight angle, optical flow and time-to-collision (from the vision-system and gimbal) and the heading angle of the vehicle (from the IMU).
\end{asm}
Now that we have defined the limitations of each component in the avoidance dynamics, we will introduce the dynamics of the variables relevant to the formulated problem. The approaching speed  and acceleration are given by:
\begin{align}
    \dot \rho(t) &= V_\mathrm{o}\left(t\right)\cos\left(\psi_\mathrm{o}\left(t\right)-\lambda \left(t\right)\right)-V_\mathrm{r}\left(t\right)\cos\left(\psi_\mathrm{r}\left(t\right) - \lambda \left(t\right)\right) \nonumber\\ \nonumber
    \ddot \rho(t) &= V_\mathrm{r}\left(t\right)\,\sin\left(\psi_\mathrm{r}\left(t\right)-\lambda \left(t\right)\right)\, \left( \dot \psi_\mathrm{r} - \dot \lambda \left(t\right)\right)\\ \nonumber
    & \, -\dot V_\mathrm{r}\left(t\right)\cos\left(\psi_\mathrm{r}\left(t\right) - \lambda \left(t\right)\right)\,  \\
     & \, -V_\mathrm{o}\left(t\right)\sin\left(\psi_\mathrm{o}-\lambda \left(t\right)\right)\left(\dot \psi_\mathrm{o}(t) - \dot\lambda \left(t\right) \right) \nonumber\\
     &\,+\dot V_\mathrm{o}\left(t\right)\cos\left(\psi_\mathrm{o}\left(t\right)-\lambda \left(t\right)\right)
     \label{eq:dynamic_LOS}
\end{align}
Also, the line-of-sight angle's dynamics are described by:
\begin{align}
\label{eq:lambdadot}
\dot \lambda(t) =  \frac{V_\mathrm{o}\,\sin\left(\psi _\mathrm{o}-\lambda \left(t\right)\right)}{\rho \left(t\right)}   - \frac{V_\mathrm{r}\left(t\right)\,\sin\left(\psi_\mathrm{r}\left(t\right) - \lambda \left(t\right)\right)}{\rho \left(t\right)}
\end{align}

As a consequence of the additive structure of the collision avoidance heading control signal, it is natural to assume bounded tracking control signal.
\begin{asm}
\label{asm:track_bound}
The heading component of the exogenous tracking controller has limited authority given by a known value $u_\mathrm{tr, max} >0$, i.e.,
\begin{equation}
    \vert u_\mathrm{tr}(\cdot) \vert \leq u_\mathrm{tr, max}.
    \label{eq:boundvehicleangrate}
\end{equation}
\end{asm}

Lastly, limited acceleration is also required to formulate a well-posed avoidance problem.

\begin{asm}
\label{asm:acl}
The accelerations of the robot and the obstacle are bounded and known.
\begin{align}
    \vert a_\mathrm{r}(t) \vert = \vert \dot V_\mathrm{r}(t) \vert \leq a_\mathrm{r, max}, \:
    \vert a_\mathrm{o}(t) \vert = \vert \dot V_\mathrm{o}(t) \vert \leq a_\mathrm{o, max}
\end{align}
\end{asm}

Thus, formally we can define the considered problem as follows.

\begin{prob}\label{prob:evasion}
Consider a vehicle that detects an uncooperative obstacle at $t=t_0$, while performing a mission. Let the dynamics of the vehicle and the obstacle satisfy the bounds given by Assumptions \ref{asm:vel},  \ref{asm:loom},  \ref{asm:track_bound} and \ref{asm:acl}. Let the information available from the obstacle satisfy Assumption~\ref{asm:sens_available}. The objective is to derive a controller that can guarantee for all $t> t_0$:
\begin{enumerate}
    \item $\rho(t) \geq \rho_{\rm{safe}}$,  where $\rho_{\rm{safe}}>0$ is the desired minimal separation in space, the violation of which constitutes a collision;
    \item $l(\rho(t),\dot \rho(t)) > -\frac{1}{\tau_\mathrm{safe}}$, where $\tau_\mathrm{safe}$ is the desired minimal separation in time-to-collision.
\end{enumerate}

\end{prob}



\begin{figure}[ht]
			\centering
			\includegraphics[width=.7\linewidth]{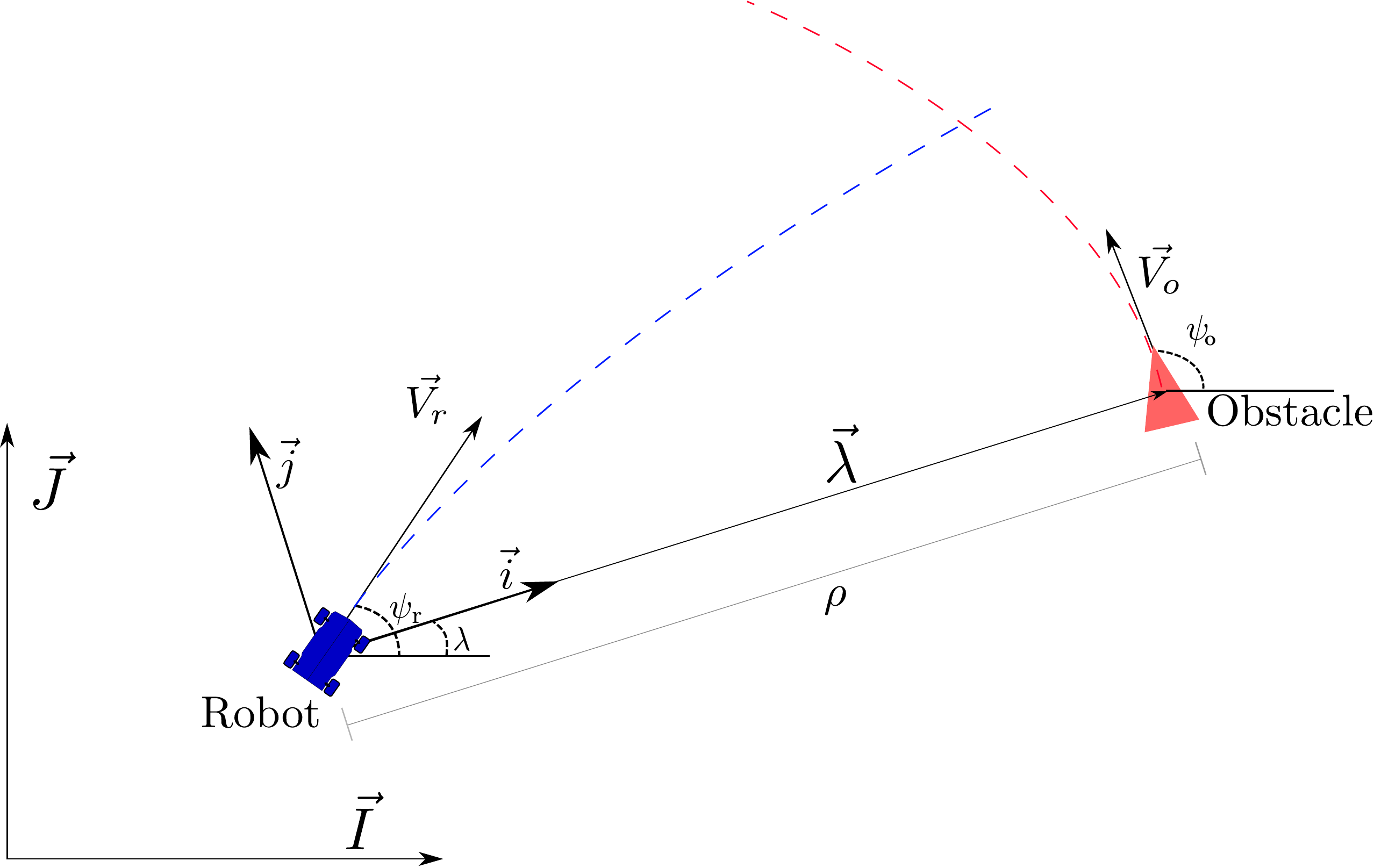}
			\caption{\small{The line-of-sight vector.}}
			\label{fig:geom}
\end{figure}

\section{Avoidance Control}
\label{sec:avoidance_control}
In this section we formulate Problem \ref{prob:evasion} as an avoidance control problem. Towards this end, we need to introduce the mathematical framework of avoidance control first developed in \cite{leitmann1977avoidance}.
Consider a continuous-time dynamical system $
\dot x(t) = f(x(t),u(t),d(t))$ with initial condition $ x(t_0) = x_0$, where  $x(t) \in \mathcal{X} \subseteq \mathbb{R}^n$ is the state vector, $u(t) \in \mathcal{U} \subseteq \mathbb{R}^m$ is the control law, $d(t) \in \mathbb{R}^d \subseteq \mathcal{D}$ is an external disturbance and $f$ is continuous in $x$.  The goal is to find a controller $u(t)$ that avoids an \textit{apriori} specified region of interest in the state space.
The set that the vehicle wants to ultimately avoid is defined as the \textit{Antitarget Region} $\mathcal{T}$, i.e., no solution of $\dot x = f(x, u, d)$ may enter $\mathcal{T}$. However, an analysis that shows avoidance of arbitrary $\mathcal{T}$ with a bounded control effort might not be possible. This can occur when the dimension of $\mathcal{T}$ is lower than the the dimension of $\mathcal{X}$. An example of this limitation is presented in \cite{leitmann1978evasion}.  To address this, one defines a higher-dimensional subset that contains $\mathcal{T}$ and is composed of all state variables required to perform the avoidance analysis, that is, $\dim\left(\mathcal{A}\right) = \dim\left(\mathcal{X}\right) $. One refers to the set $\mathcal{A} \supseteq \mathcal{T}$  as the \textit{Avoidance Region}, where the states are  not allowed to enter. It is important to notice that once the set $\mathcal{A}$ is avoided,  then $x(t) \notin \mathcal{T}$ for any time $t$. Finally, define the closed set $\Delta_\mathcal{A}$, where $\mathcal{A} \subset \Delta_\mathcal{A} \subseteq\mathcal{X}$. Then, we call the set $\Omega_\mathcal{A} =  \Delta_\mathcal{A} \setminus \mathcal{A}$  the \textit{Conflict Region}.

Given these definitions, we state the avoidance control objective:  given any disturbance $d(\cdot) \in\mathcal{D}\subset \mathbb{R}^d$, the regions $\mathcal{T}, \mathcal{A},$ and $\Omega_{\mathcal{A}}$,  derive a control $u$ such that  a trajectory $x(t)$ starting at $x(t_0) = x_0 \in \Omega_\mathcal{A}$ will never enter $\mathcal{A}$.

\section{Avoidance Regions}

In this section the anti-target, avoidance and conflict regions for our problem are presented. With the definition of those regions we will derive conditions for which the bound of Assumptions \ref{asm:vel} - \ref{asm:loom} must satisfy. To this end, consider the state space  $x(t)= [x_1(t), \, x_2(t),\, x_3(t),\, x_4(t),\, x_5(t),\, x_6(t)]^\intercal $ that captures the dynamics of inverse of time-to-collision (loom) and the inverse of distance, such that%
\begin{align}
x_1(t) &=\frac{\dot \rho(t)}{\rho(t)}, \quad x_2(t) = \lambda(t),\quad x_3(t) = \dot \lambda(t), \nonumber \\ 
x_4(t) &= \frac{1}{\rho(t)}, \quad x_5(t) = \psi_r(t), \quad x_6(t) = V_r(t)
\end{align}%
Note that, later in this paper, we will prove that $\rho$ is bounded from below, and hence $\frac{1}{\rho}$ will be bounded. Using the line-of-sight dynamics \eqref{eq:dynamic_LOS} and \eqref{eq:lambdadot}, we can write the dynamics of our system as follows:
\begin{align} 
\label{eq:statespace}
\dot x_1(t) &=\frac{\ddot \rho(t) \rho(t) - \dot \rho^2(t)}{\rho^2(t)}= \frac{\ddot \rho(t)}{\rho(t)} - \left(\frac{\dot \rho(t)}{\rho(t)}\right)^2 = x_4 \ddot \rho(t) - x_1^2(t),\nonumber \\
\dot x_2(t) &= x_3(t), \quad  \dot x_3(t) = \ddot \lambda(t), \quad \dot x_4(t) = -x_4(t) x_1(t), \nonumber\\
\dot x_5(t) &= u(t),\quad
\dot x_6(t) = u_v(t),
\end{align}
where $x_1 \in \mathbb{R}_{\leq 0}$, $x_2 \in [-\pi, \; \pi]$, $x_3\in \mathbb{R}$, $x_4 \in \mathbb{R}_{\geq 0}$, $x_5 \in [-\pi, \; \pi]$, $x_6\in [V_\mathrm{r, min}, V_\mathrm{r, max}]$, and 
\resizebox{\linewidth}{!}{
  \begin{minipage}{\linewidth}
  \begin{align*}
\ddot \rho   &=x_6\left(t\right)\sin\left(x_5\left(t\right)-x_2 \left(t\right)\right)\left( u(t) - x_3 \left(t\right)\right)\nonumber -u_v\left(t\right)\cos\left(x_5\left(t\right) - x_2 \left(t\right)\right)  \\
     &-V_\mathrm{o}\left(t\right)\sin\left(\psi_\mathrm{o}-x_2 \left(t\right)\right)\left(\dot \psi_\mathrm{o}(t) - x_3 \left(t\right) \right) +\dot V_\mathrm{o}\left(t\right)\cos\left(\psi_\mathrm{o}\left(t\right)-x_2 \left(t\right)\right)
\end{align*} \end{minipage}}
From this point forward the time dependence $(t)$ will be dropped for notation convenience.


To fulfill our objectives we want $\tau > \tau_\mathrm{safe}$ and $\rho = \frac{1}{x_4} > r\geq \rho_{\text{safe}}$. Thus, we have to create an equivalent region on $x_1$ defined by loom. The following anti-target set represents such region: 
\[
\mathcal{T} = \left\{ x_1 \leq -\frac{1}{\tau_\mathrm{safe}}\right\} \cup \left\{ \frac{1}{x_4} \leq r\right\}.
\]
Note that in this case, the line-of-sight angle rate $x_3$ given in \eqref{eq:lambdadot} is bounded by $L >0$ when the obstacle is outside of the anti-target region, i.e. $x \in \mathcal{T}^\complement$. The bound $L$ is given by:
\begin{align*}
\vert x_3 \vert &= \vert x_4 \big( V_\mathrm{o} \sin(\psi_\mathrm{o} - x_2) - x_6 \sin(x_5 - x_2) \big) \vert \\
&\leq  \frac{V_\mathrm{r, max} + V_\mathrm{o, max}}{r} = L.
\end{align*}

With the desire to achieve avoidance with a bounded controller, an avoidance set larger than $\mathcal{T}$ needs to be defined. This makes sure the controller has enough time to evade the obstacle given that there is a limited evading control resource. To illustrate the need of a larger avoidance set, consider the case of an obstacle appearing very close to the boundary of $\mathcal{T}$ such that it starts too close and approaches too fast; in this case, there is no feasible avoidance strategy without using large control signals. We chose to represent the extra time as a relationship between the line-of-sight angle $x_2$ and the robot's heading angle $x_5$. Define $\Delta_t(x)$ as a state-dependent function that represents the extra time-to-collision chosen to create a safety buffer such that
\begin{equation}
\Delta_t(x) = \frac{\pi - \vert x_5 - x_2\vert}{\beta},
\end{equation}
where $\beta>0$ is a design parameter. Larger the $\beta$,  less "safety cushion" is given around $\mathcal{T}$, and therefore a higher bound on the control effort is needed. 
The intuition is that if $x_5 - x_2 =0$, the robot is facing the obstacle and it needs the largest amount of time (given the bound on the control) to turn around to larger angles of $x_5-x_2$. When $x_5 - x_2 $  approaches $\pi$ or $-\pi$, the robot is "running away" from the obstacle and doing its best to avoid collision and increase the time-to-collision. 
Now we can build the avoidance set $\mathcal{A}$ using $\Delta_t$. Similar to the anti-target set, the avoidance region is the union of a region for $x_1$ and another for $x_4$, that is, $\mathcal{A} = \mathcal{A}_1\cup\mathcal{A}_2$. Define $\mathcal{A}_1$ as:
\[
\mathcal{A}_1 = \left\{ x_1 \leq -\frac{1}{\tau_\mathrm{safe} + \Delta_t(x)}\right\}.
\]
The worst case scenario formulation is the basis to create a similar buffer region $\mathcal{A}_2$ for the minimal distance component. While the robot is maneuvering in $\Delta_t$ time, we compute how much distance is traveled by the robot to get to the safest configuration $\vert x_5 - x_2 \vert= \pi$:
\begin{equation*} S_r = \int_t^{t+\Delta_t} x_6 \cos(x_5 - x_2) dt.
\end{equation*}
At the same time, the obstacle travels
\begin{equation*}
S_o = \int_t^{t+\Delta_t} V_\mathrm{o} \cos (\psi_o(t) - x_2) dt,\end{equation*}
hence, the total separation distance is
\begin{equation*}
S = S_r + S_o \leq (x_6 + V_\mathrm{o, max})\Delta_t.
\end{equation*}
Therefore, the avoidance set is
\begin{align*}
\mathcal{A} &= \left\{ x_1 \leq -\frac{1}{\tau_\mathrm{safe} + \Delta_t(x)}\right\} \bigcup \\
&\left\{\frac{1}{x_4} \leq r + (x_6 + V_\mathrm{o, max})\Delta_t(x) \right\}.
\end{align*}
There are many ways to design a suitable conflict region and for this purpose we chose:
\begin{align}
    \Omega_\mathcal{A} = &\left\{x\in \mathcal{X} \left\vert  -\frac{1}{\tau_\mathrm{safe} + \Delta_t} < x_1 \leq -\frac{\beta}{\gamma}, \ \right. \right. \nonumber \\
    &\left. r + (x_6 + V_\mathrm{o, max})\Delta_t<\frac{1}{x_4} \leq \omega  \right\},
\end{align}
where we define the constant $\gamma = \tau_\mathrm{safe} \beta + \pi$, and $\omega$ is a design parameter that must be chosen so that 
\begin{equation}
    \label{eq:omega} 
\omega > r + (V_\mathrm{r, max} + V_\mathrm{o, max} )\frac{\pi}{\beta}. 
\end{equation}
\begin{figure}
    \centering
    \includegraphics[width=.25\textwidth]{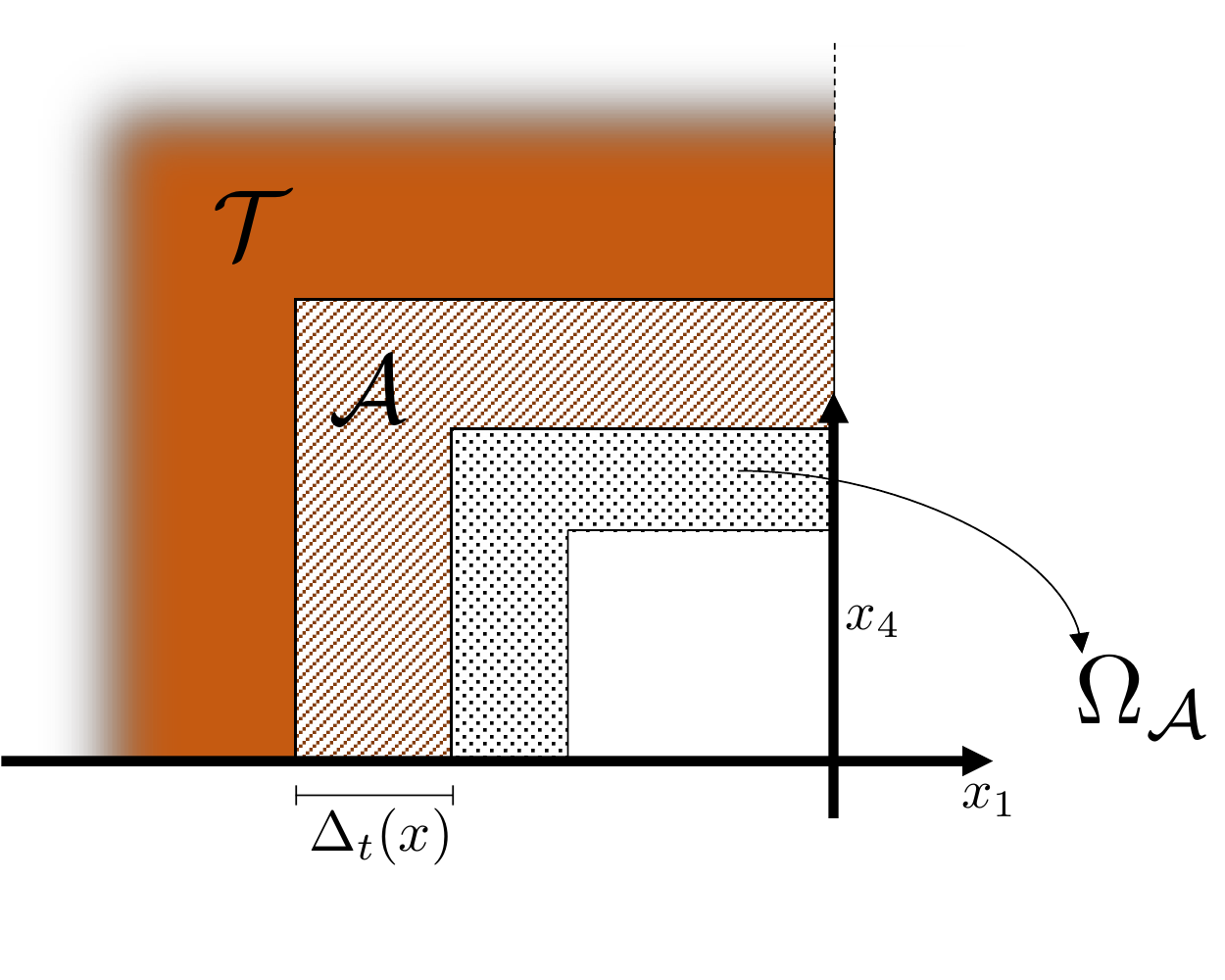}
    \caption{\small{A 2D slice for a given $\Delta_t(x)$ of the anti-target, avoidance and conflict regions.}}
    \label{fig:slice_regions}
\end{figure}

\section{Main Result}
Recall that the heading control law that drives the angular rate of the robot is given by 
\begin{equation}
\label{eq:control}
    u(t) =u_\mathrm{ca}(t) + u_\mathrm{tr}(t),
\end{equation}
where the outer loop exogenous controller is $u_\mathrm{tr}$ that satisfies Assumption \ref{asm:track_bound}. We present the following avoidance component of the heading controller:
\begin{align}
\label{eq:ca_control}
    u_\mathrm{ca} &= \alpha_0(x) \Big( \frac{\gamma^2}{\beta} \left(x_1^2  + \alpha_1(x) \right) + \alpha_2(x) \Big) + x_3,
\end{align}
 where the state dependent terms $\alpha_0$, $\alpha_1$ and $\alpha_2$ are defined as:
 \begin{align*}
     \alpha_0(x)&=\begin{cases}
     \: 1 & \text{if } x_5-x_2 \geq 0\\
     -1 & \text{if } x_5-x_2 < 0
     \end{cases}\\
\alpha_1(x) &=\big(k x^2_1 + x_6  \max\{0,- x_3\sin(x_5-x_2)\}  \\
& - \min \{ 0, -a_\mathrm{r} \cos(x_5 - x_2) \}\big) \frac{1}{ r + (x_6 + V_\mathrm{o, max} ) \Delta_t(x) }\\
\alpha_2(x)& = -2\beta\omega\frac{x_1}{x_6+V_{\mathrm{o},\max} },
\end{align*}
with $k>0$ being a scalar control gain. 



Before stating the main result, let's explicitly introduce the bounds defined in Assumptions \ref{asm:vel} - \ref{asm:loom} that must be satisfied for us to show that under this constrained controller our objective of avoidance is achieved.

The limited control authority of the exogenous tracking effort defined in Assumption \ref{asm:track_bound} is dependent on the control gain $k$, the design parameters $\beta$, $\tau_\mathrm{safe}$, and the minimum speed of the robot and maximum speed of the obstacle:
\begin{align}
\label{eq:track_bound}
    u_\mathrm{tr, max} & < \min\left\{ \frac{\beta^2\omega}{V_\mathrm{r, max} \gamma}, \;  \beta\right\},
\end{align}
while the velocity controller must respect the following acceleration limit:
\begin{equation}
    a_{\max} < \frac{\beta^2 \omega}{\gamma \pi}
    \label{eq:accel_bound}
\end{equation}

We can also show the existence of a bound for $u_\mathrm{ca}$, by looking for an upper bound of each component of $u_\mathrm{ca}$ and finding a conservative upper bound $\vert u_\mathrm{ca} \vert \leq u_\mathrm{max}$\footnote{$\max_{z}\{g(z) + h(z)\} \leq \max_{z}\{g(z)\} + \max_{z}\{h(z)\}$}:
\begin{align}
\label{eq:u_bound}
u_{\max}& = \frac{\gamma^2}{\beta}\left( k\left(\frac{1}{\tau_r}\right)^2 + \frac{V_\mathrm{r, max} L}{r + (V_\mathrm{r, min} + V_\mathrm{o, max}) \frac{\pi}{\beta} }\right)\nonumber\\
&\;+ 2\frac{\gamma^2}{V_\mathrm{r, min} \tau_r} + L 
\end{align}

\begin{rem}
While \eqref{eq:u_bound} shows that the proposed avoidance control is bounded, it is a conservative bound. The system can't achieve a location of the state space, where all the terms will have the values as in \eqref{eq:u_bound}, and finding the lowest upper bound for $u_\mathrm{ca}$ would involve solving a nonlinear optimization problem constrained to \eqref{eq:statespace}. Nevertheless this bound gives some insight on how to choose $\tau_\mathrm{safe}$ and $r$ and its relationship to $V_\mathrm{r, min}$, $V_\mathrm{r, max}$, $V_\mathrm{o, max}$ and $k$. 
\end{rem}

 The evading maneuver starts as soon as an object is detected at time $t_0$. It is assumed that at this moment the relative position and velocity between the robot and the object is such that $x \notin \mathcal{A}$. In other words, there is an initial space and time-to-collision separation.

\label{sec:result}

\begin{thm}[Collision Avoidance]
\label{lemma}
Consider the system~\eqref{eq:statespace}, that describes the kinematics of a robot or vehicle travelling with speed $V_\mathrm{r}(t)$ and a non-cooperative moving obstacle with speed $V_\mathrm{o}(t)$ and heading $\psi_\mathrm{o} (t)$, both satisfying Assumption~\ref{asm:vel}. 

We assume that the collision maneuver starts as soon as the obstacle is detected at time $t = t_0$, such that $x(t_0) \notin \mathcal{A}$. Consider the control law \eqref{eq:control} with an exogenous tracking controller $u_\mathrm{tr}$ that satisfies Assumption~\ref{asm:track_bound} and the inequality in~\eqref{eq:track_bound}, and the avoidance control $u_\mathrm{ca}$ given by \eqref{eq:ca_control}. Let $k$ satisfy
\begin{align}
\label{eq:kgain}
    k &\geq  \left(\frac{\gamma}{\beta} \right)^2 \left(V_\mathrm{o, max}\dot \psi_\mathrm{o, max}  +a_\mathrm{o, max}\right).
\end{align}
Then, $\mathcal{A}$ is avoidable, and the trajectory $x(t)$ never reaches the anti-target set $\mathcal{T}$ for any $t \geq t_0$.
\end{thm}

\begin{proof}

Consider the two functions $A_1(x)$ and $A_2(x)$ that are designed to achieve their zero at $\{ x_1 = -\frac{1}{\tau_r + \Delta t}\}$ and $ \left\{\frac{1}{x_4}  = r + (x_6 + V_\mathrm{o, max} ) \Delta_t \right\}$ respectively: 

\begin{align}
    A_1(x) &= x_1 + \frac{\beta}{\gamma - \vert x_5 - x_2\vert} \\
    A_2(x) &=\frac{1}{x_4} - \left( r + (x_6 + V_\mathrm{o, max} )\left(\frac{\pi - \vert x_5 - x_2\vert}{\beta} \right) \right)
\end{align}

Notice that $A_1(x)$ and $A_2(x)$ are positive in $\mathcal{A}^\complement$. Define the Lyapunov like $C^1$-function $V(x)$, for $x \in \Omega_\mathcal{A}$, as:
 
\begin{align}
V(x) &= \log(A_1(x) + 1) + \log (A_2(x) + 1)
\end{align}
The time derivative of $V$ becomes
\[
\dot V = \frac{1}{A_1 + 1} \dot A_1 + \frac{1}{A_2 + 1} \dot A_2
\]

As shown in \cite{leitmann1977avoidance,leitmann1978evasion}, if $\dot V \geq0$ for all $x \in \Omega_\mathcal{A}$, then the collision avoidance is guaranteed. Therefore, to satisfy the avoidance condition in Theorem $\ref{lemma}$, it suffices to show that $u_\text{ca}$ guarantees $\dot A_1 \geq 0$ and $\dot A_2 \geq 0$ for all $x \in \Omega_\mathcal{A}$, given any $V_\mathrm{o}, \psi_\mathrm{o}$, $a_\mathrm{o}$ and $u_\text{tr}$ that satisfy Assumptions \ref{asm:vel}, \ref{asm:track_bound} and \ref{asm:acl}. 

First, consider the case where $(x_5 - x_2) >0$ and recall that $x\in \Omega_\mathcal{A}$. Then,
\begin{align*}
    \dot A_1 &= \dot x_1 + \frac{\beta \left( \dot x_5 - \dot x_2 \right)}{(\gamma - (x_5 - x_2))^2}= x_4 \ddot \rho - x_1^2 + \frac{\beta \left( u -x_3 \right)}{(\gamma - (x_5 - x_2))^2}
\end{align*}
Notice that 
\begin{align}\frac{\beta \left( u -x_3 \right)}{(\gamma - (x_5 - x_2))^2} \geq \frac{\beta \left( u_\mathrm{ca} -x_3 \right)}{\gamma^2} + \frac{\beta  u_\mathrm{tr} }{(\gamma - (x_5 - x_2))^2},
\end{align}
and $u_\mathrm{ca} > x_3$, which holds from equation \eqref{eq:ca_control} by using the fact that $x\in\Omega_\mathcal{A}$ implies $x_1<0$. Thus,
\begin{align*}
    \dot A_1 & \geq x_4 \ddot \rho - x_1^2 +\frac{\beta \left( u_\mathrm{ca} -x_3 \right)}{\gamma^2} + \frac{\beta  u_\mathrm{tr}}{(\gamma - (x_5 - x_2))^2}\\
    & = x_4(t) \big[x_6\sin\left(x_5-x_2 \right) \left( u - x_3 \right)\nonumber -u_v\cos\left(x_5 - x_2 \right)  \\
     & \, -V_\mathrm{o}\sin\left(\psi_\mathrm{o}-x_2 \right)\left(\dot \psi_\mathrm{o} - x_3  \right) +\dot V_\mathrm{o}\cos\left(\psi_\mathrm{o}-x_2 \right)\big]  \nonumber \\
    & \; - x_1^2 + \frac{\beta \left( u_\mathrm{ca} -x_3 \right)}{\gamma^2} + \frac{\beta  u_\mathrm{tr} }{(\gamma - (x_5 - x_2))^2}
\end{align*}

Moreover, by using the fact that if $x\in\Omega_\mathcal{A}$ then $x_1\leq -\frac{\beta}{\gamma}$ and that $u_{\text{tr}}$ satisfies condition \eqref{eq:track_bound}, one can show that $u>x_3$. Therefore, for $x_5-x_2>0$, we have  $x_6 \sin(x_5-x_2) (u-x_3) \geq 0$. So it follows that:

\begin{align*}
    \dot A_1 & \geq x_4 \big[ -u_v\cos\left(x_5 - x_2 \right)  +V_\mathrm{o}\sin\left(\psi_\mathrm{o}-x_2 \right) x_3 \\ & \;-V_\mathrm{o}\sin\left(\psi_\mathrm{o}-x_2 \right)\dot \psi_\mathrm{o}  +\dot V_\mathrm{o}\cos\left(\psi_\mathrm{o}-x_2 \right)\big] \\
    & \; - x_1^2 + \frac{\beta \left( u_\mathrm{ca} -x_3 \right)}{\gamma^2} + \frac{\beta  u_\mathrm{tr} }{(\gamma - (x_5 - x_2))^2}\\
    &\geq x_4 \big[ -u_v\cos\left(x_5 - x_2 \right)  +V_\mathrm{o}\sin\left(\psi_\mathrm{o}-x_2 \right) x_3 \\ & \;-\left(V_\mathrm{o, max}\dot \psi_\mathrm{o, max}  +\dot V_\mathrm{o, max}\right)\big] \\
    & \; - x_1^2 + \frac{\beta \left( u_\mathrm{ca} -x_3 \right)}{\gamma^2} + \frac{\beta  u_\mathrm{tr} }{(\gamma - (x_5 - x_2))^2}
\end{align*}
where $u_v$ is given by a trajectory tracking controller and is guaranteed to keep the velocity $V_\mathrm{r}$ in $ [V_\mathrm{r, \min}, V_\mathrm{r, \max} ]$. From this moment forward we shall consider $ u_v = a(t)$ given by an external velocity controller and use $a$ to maintain the notation brief.

The term $V_\mathrm{o} \sin( \psi_\mathrm{o} - x_2) x_3 $ could also be bounded by $-V_\mathrm{o, max} L$, but we chose to  make a less conservative substitution  looking into the dynamics of $x_3$. Using our state space variables, equation \eqref{eq:lambdadot} becomes
\begin{align*}
   x_3 &= x_4 \left(V_o \sin(\psi_o - x_2) - x_6 \sin(x_5 - x_2) \right).
\end{align*}
Multiplying both sides by $x_3$ and recalling that $x_4>0$ gives
\begin{align}
V_\mathrm{o} \sin( \psi_o - x_2) x_3 \geq   x_6 x_3 \sin(x_5 - x_2),
\end{align}
which allows  to continue with
\begin{align*}
    \dot A_1 &\geq x_4 \big[ x_3 x_6 \sin(x_5 - x_2)- a \cos(x_5 - x_2) \\ & \;-\left(V_\mathrm{o, max}\dot \psi_\mathrm{o, max}  +\dot V_\mathrm{o, max}\right)\big] \\
    & \; - x_1^2 + \frac{\beta \left( u_\mathrm{ca} -x_3 \right)}{\gamma^2} + \frac{\beta  u_\mathrm{tr} }{(\gamma - (x_5 - x_2))^2}.
\end{align*}
Here we substitute the controller by the one defined in \eqref{eq:ca_control}:
\begin{align*}
    \dot A_1  &\geq \frac{x_6  \max\{0, -x_3\sin(x_5-x_2)\}}{ r + (x_6 + V_\mathrm{o, max}) \Delta_t(x) } + x_3 x_4 x_6 \sin(x_5 - x_2)  \\
    &+ \left(-\frac{\min \{ 0, -a \cos(x_5 - x_2)}{ r + (x_6 + V_\mathrm{o, max})\Delta_t(x) } - x_4 a \cos(x_5 - x_2)  \right) \\
    & +  \frac{k x_1^2}{ r + (x_6 + V_\mathrm{o, max})\Delta_t(x) }-x_4\Big(V_\mathrm{o, max}\dot \psi_\mathrm{o, max} +a_\mathrm{o, max}\Big)\\
    &   -\frac{\beta^2}{x_6\gamma^2}\omega x_1 + \frac{\beta  u_\mathrm{tr} }{(\gamma - (x_5 - x_2))^2}
    \geq 0 
\end{align*}
Since we are only interested in evaluating $\dot V$ for $x\in \Omega_\mathcal{A}$, we have that $x_4 < \frac{1}{r + (x_6 + V_\mathrm{o, max}) \Delta_t}$ and $x_1 \leq \frac{-\beta}{\gamma} $. Therefore, with $k$ given by \eqref{eq:kgain}, we have that $\dot A_1 \geq 0$ if  
\begin{align}\vert u_\mathrm{tr} \vert &\leq \frac{\beta^2 \omega}{V_\mathrm{r, max}\gamma},\end{align} which is satisfied by  \eqref{eq:track_bound}.

Now, we will show that $\dot A_2\geq 0$. Notice that, from our state space definition, we have $\frac{d}{dt}\left(\frac{1}{x_4}\right) = \frac{d\rho(t)}{dt} = \dot \rho$ and $x_1(t) =\frac{\dot \rho(t)}{\rho(t)}$ or equivalently $\dot \rho(t) = \frac{x_1}{x_4}$. So $\dot A_2$ becomes:
\begin{align*}
\dot A_2 = &\frac{x_1}{x_4}+\frac{x_6+V_{\text{o},\max} }{\beta}(u-x_3)+\frac{\dot x_6}{\beta}(\pi-(x_5-x_2)) \\
=&\frac{x_1}{x_4}+\frac{x_6+V_{\text{o},\max} }{\beta}(u-x_3)+u_v\Delta_t(x).
\end{align*}
Substituting by the controller defined in equations \eqref{eq:control} and \eqref{eq:ca_control}, gives
\begin{align*}
\dot A_2 \geq& \frac{x_1}{x_4}+ \frac{\gamma^2  (x_6+V_{\text{o},\max})}{\beta^2}(x_1^2 +\alpha_1(x))-2\omega x_1  \\
&+\frac{x_6+V_{\text{o},\max} }{\beta}u_{\text{tr}}+u_v\Delta_t
\end{align*}
Note that $\alpha_1(x)\geq 0$. Thus, by rearranging the terms we have 
\begin{align*}
\dot A_2 &\geq \left(\frac{1}{x_4} - 2\omega\right)x_1+ \frac{x_6+V_{\text{o},\max}}{\beta}\left(u_{\text{tr}}+\frac{\gamma^2 x
_1^2}{\beta}\right)+u_v\Delta_t. 
\end{align*}
Since we are only interested in evaluating $\dot V$ in the conflict region where $x \in \Omega_A$, we know that $x_1\leq-\frac{\beta}{\gamma}$ and $\frac{1}{x_4}\leq\omega$. Then, the first term of $\dot A_2$ can be shown to be non-negative, as follows:
\begin{align*}
\left(\frac{1}{x_4} - 2\omega\right)x_1 \geq \left\vert\frac{1}{x_4} - 2\omega \right\vert\frac{\beta}{\gamma} \geq \frac{\beta\omega}{\gamma}.
\end{align*}
Moreover, we have $|u_\text{tr}|<\beta$ and  $|u_v|<a_{\max}<\frac{\beta^2\omega}{\gamma\pi}$. Thus,

\begin{align*}
\dot A_2 &\geq  \frac{\beta\omega}{\gamma} + \frac{x_6+V_{\text{o},\max}}{\beta}\left(u_{\text{tr}}+\beta\right)+u_v\Delta_t \\
& \geq \left(\frac{\beta\omega}{\gamma} +  \frac{u_v(\pi - |x_5 - x_2|)}{\beta} \right) + \frac{x_6+V_{\text{o},\max}}{\beta}\left(u_{\text{tr}}+ \beta \right) \\
&\geq0
\end{align*}

Hence, we conclude that for $x_5 - x_2>0$ we have that $\dot V \geq 0$. Since our problem is symmetric around the line-of-sight $x_2$, an analogous argument can show that for $x_5 - x_2 < 0 $, $\dot V \geq 0$.

At the surface $\left\{x\in \mathcal{X} | \ \ x_5-x_2 = 0\right\}$ the function $\dot V$ does not have a defined sign. However, this does not pose an issue in our analysis, since at this measure zero surface, $\dot V$ is non-negative at both sides of this surface. Hence, collision avoidance can still be guaranteed using our analysis.
\end{proof}

\section{Experimental Results}
This section presents experimental results performed with ground robots to validate the methodology. We consider two distinct scenarios of an obstacle approaching the evader. The evading robot follows a smooth polynomial trajectory across the arena, while the obstacle approaches from different directions. The goal is to show the efficacy of the avoidance control in the presence of sensor noise and time-delays inherent in practical applications. Moreover, the evader robot used is not a perfect unicycle model, further validating the robustness of the control law. In what follows, the system architecture and the indoor facility used to conduct the experiments are described, followed by a detailed discussion of the test results.

The platform used in the tests is the Jackal differential drive robot from Clearpath Robotics shown in Figure \ref{fig:exp1video}. Two robots were used, one as the evader and the other as the moving obstacle.  The robot with the orange cylindrical object is the obstacle to be avoided.

To localize both the evader and the obstacle, we used the precise localization technology in Virtual Reality headsets. By using inexpensive HTC Vive Trackers, positions, orientations and linear velocities were obtained for both agents.  Since the development of a vision system is outside the scope of this work, the time-to-collision and line-of-sight rate are computed from position and velocities, thus simulating the vision system. We emphasize that the framework does not use distance in feedback; distance is only used as a way replace the unavailable vision module. The obstacle's heading angle, velocity and distance to the evader are never used explicitly for avoidance and only used to calculate time-to-collision, line-of-sight and line-of-sight rate. The relationships for $\dot \lambda$ in \cite{marinho2018guaranteed} and \eqref{eq:loom} are used to calculate the information that would otherwise be obtained from a vision system.

The Vive trackers utilize multiple infra-red (IR) sensors and an IMU to perform the localization task. A base station emits IR light that sweeps the room at a frequency of 50Hz, hitting the IR sensors on the tracker at different time intervals. The size of these time intervals is a function of the relative position of the sensor to the base station. The tracker sends the data through Bluetooth to the HTC Vive system and filters it to create a precise estimate of the tracker's position and velocity.

The control loop runs on a centralized computer with Simulink and ROS (Robot Operating System). Simulink obtains the tracker's information through a Python script connected to the HTC Vive software. Then the script publishes the tracker's pose to a ROS topic.   The robot receives the heading rate $u(t)$ and speed $V_\mathrm{r}(t)$ published by ROS in Simulink, where the control algorithm is executed. A lower-level controller onboard the Jackal Robot commands the wheel motors to ensure that the heading rate and speed commands are executed. A user manually controls the obstacle that is also equipped with a Vive Tracker. Figure \ref{fig:setup} depicts the described setup.
\begin{figure}[ht]
    \centering
    \includegraphics[width=0.5\linewidth]{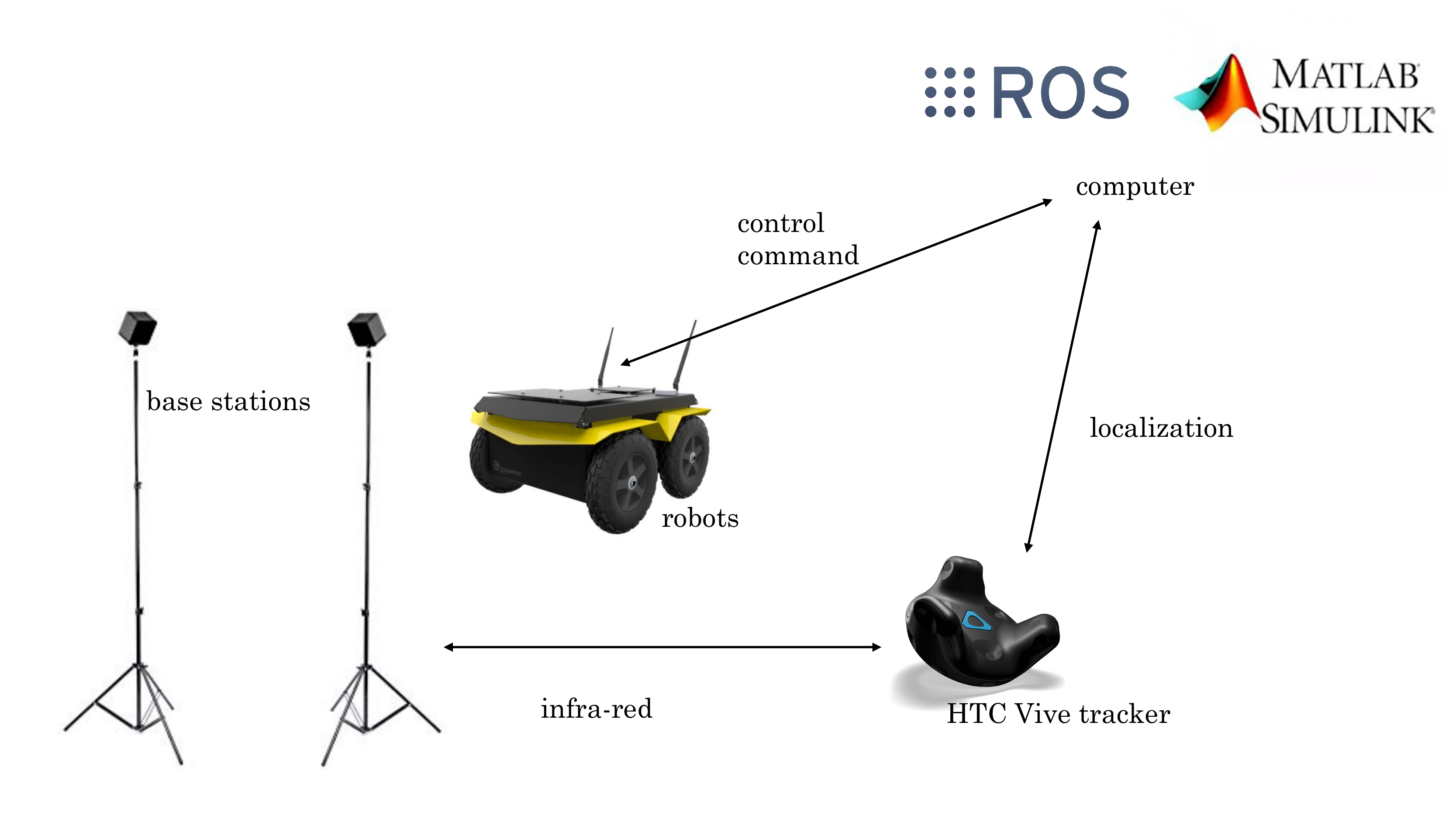}
    \caption{\small{Diagram of the hardware and software setup.}}
    \label{fig:setup}
\end{figure}

In all scenarios the robot follows the desired trajectory with the path following controller inspired by \cite{mastellone2008formation}:
\begin{align}
    u_\mathrm{tr} &=-3\sin(\psi_\mathrm{e}), \; V_\mathrm{r} = 1.5 D \cos(\psi_\mathrm{e}),
        \label{eq:tracking_controller}
\end{align}
where $\psi_\mathrm{e} = \psi_r - \mathrm{atan2}(y_e, \;x_e)$, $[x_e,\; y_e]^\top = [x_d - x_r, \; y_d - y_r]^\top$ and $D=\sqrt{x^2_e + y^2_e}$. We saturate both tracking controllers so that  $u_\mathrm{tr, max} = 1.0$ rad/s, $V_\mathrm{r, min} = 0.2$ m/s, $V_\mathrm{r, max} = 0.5$ m/s, and $a_\mathrm{r, max}= 3.5$ m/s$^2$. The obstacle is manually controlled through a joystick, and in our control design we assume that the maximum speed that the robot can achieve is $V_\mathrm{o, max}= 2.0$ m/s with $\dot \psi_\mathrm{o, max} = 0.5$rad/s. The goal is to maintain a safety distance of $r = 0.5$ m, that is consistent with the size of the Jackal robots, and guarantee a  minimum time-to-collision of $\tau_\mathrm{safe} = 0.5$ s. To achieve this, the avoidance controller was designed by choosing $\beta=6.3$ and control gains $\omega=1.75$ and $k=1$, which satisfy the sufficient condition for collision avoidance in \eqref{eq:omega} and \eqref{eq:kgain}.
\vspace{-.35cm}
\subsection{Scenario 1}
In the first scenario the obstacle heads directly towards the evader and the collision avoidance strategy successfully avoids the obstacle. An overlap of video frames is illustrated in Figure \ref{fig:exp1video} which shows the time lapse of this experiment. To visualize the time evolution of the maneuver, transparency was used in the superimposed images. More transparent images of the robot represent the past, while the more opaque objects show where the agents are later in the mission.
\begin{figure}[ht]
    \centering
    \includegraphics[width=0.55\linewidth]{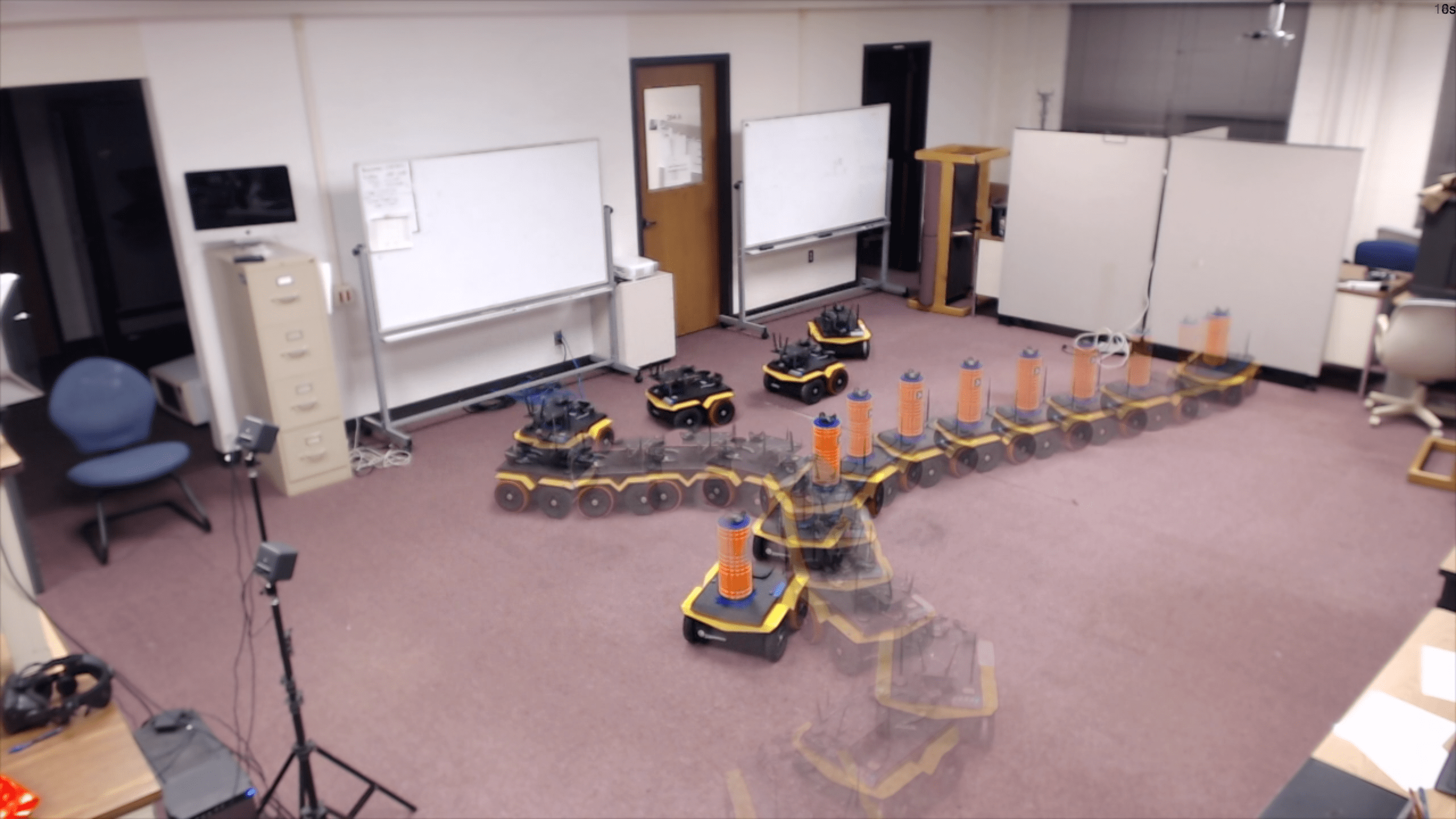}
    \caption{\small{Superimposed images from video of the fist experiment. The transparency represents the time evolution.}}
    \label{fig:exp1video}
\end{figure}

As seen in Figure \ref{fig:exp1}, at approximately 3.5 seconds, the obstacle starts moving, which generates a large enough loom (or small enough time-to-collision), and the evading maneuver starts. At this moment, the collision avoidance controller overcomes the tracking heading command. As a result, the robot starts to turn away from the planned path. Notice that at 7.5 seconds, the obstacle is close to the virtual target, and if the robot was only following the target, there would have been a collision. Thanks to the algorithm, the robot has moved away from this unsafe location.  At $t=9.5$s, the plots in Figure \ref{fig:exp1_s} show that the obstacle is no longer approaching the obstacle. This implies that the distance is increasing and the time to collision assumes large values. It follows that the robot automatically turns towards the virtual target, accelerates to catch-up,  and resumes the mission.

\begin{figure}[ht]
    \centering
    \includegraphics[width=0.6\linewidth]{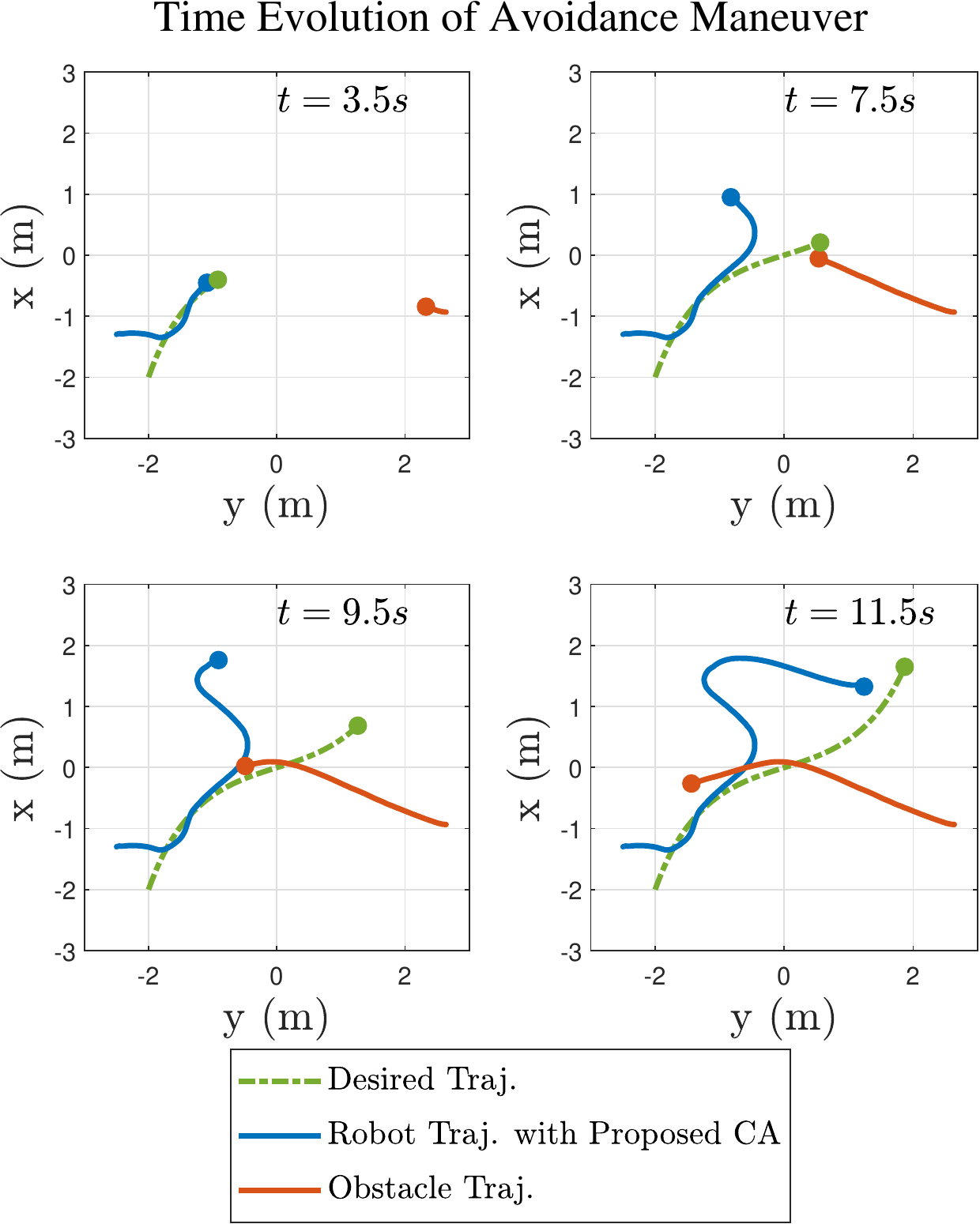}
    \caption{\small{Time evolution of robot's evading maneuver.}}
    \label{fig:exp1}
\end{figure}

\begin{figure}[ht]
	\centering
	\resizebox{.7\linewidth}{!}{
	{\includegraphics[width=0.4\linewidth]{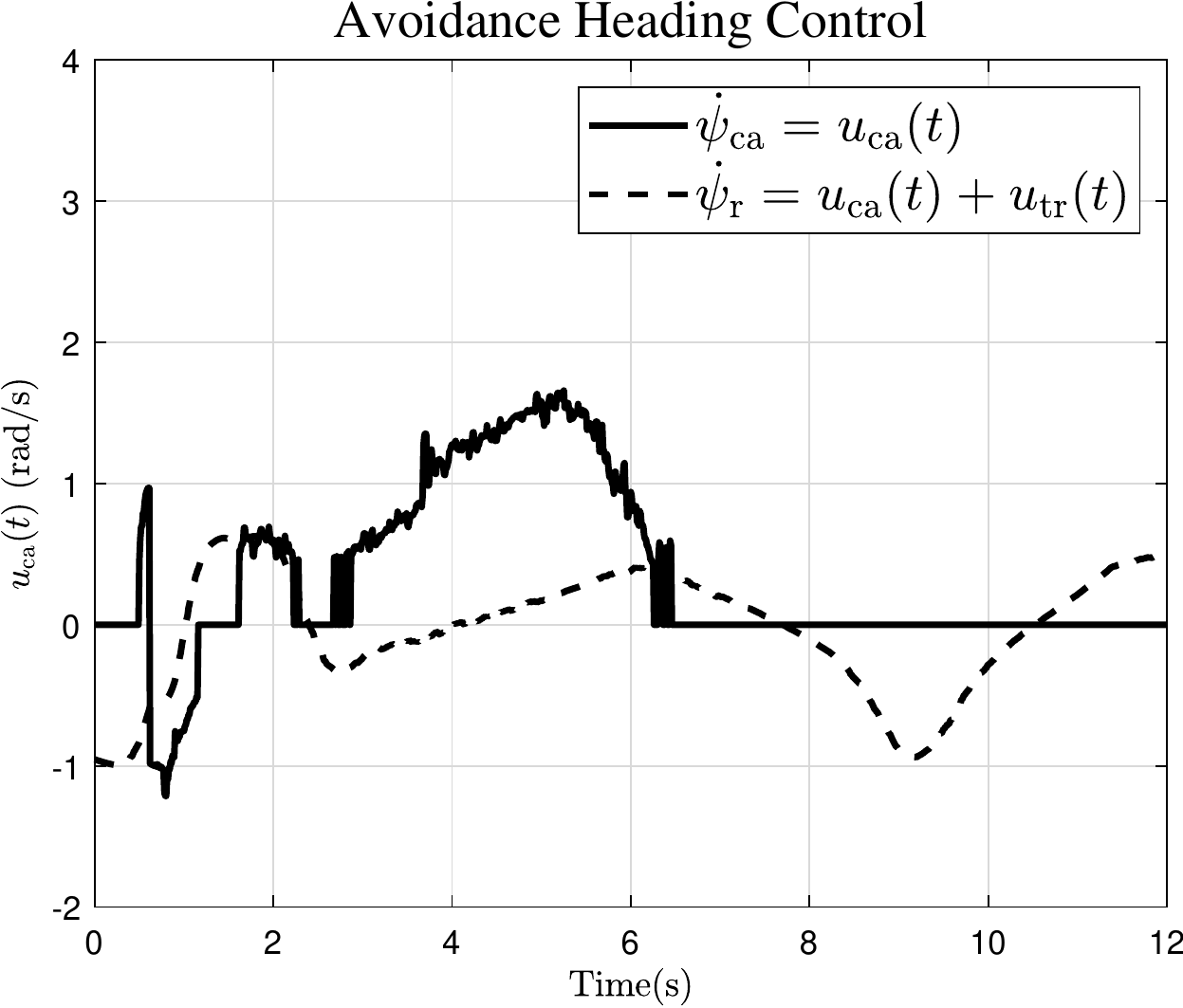}
    \label{fig:exp1_uca}}
	{\includegraphics[width=0.4\linewidth]{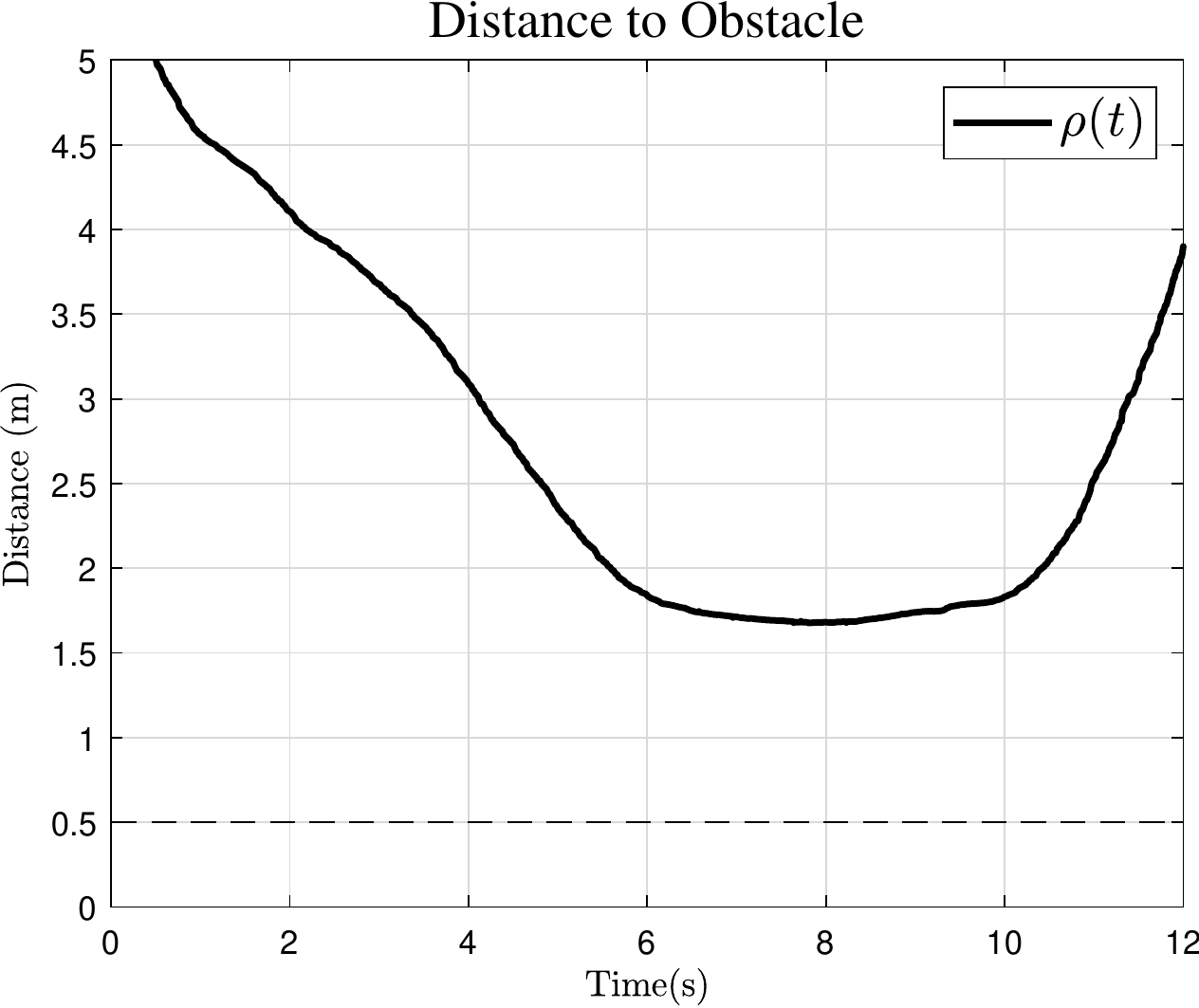}
    \label{fig:exp1_rho} }
 {\includegraphics[width=0.4\linewidth]{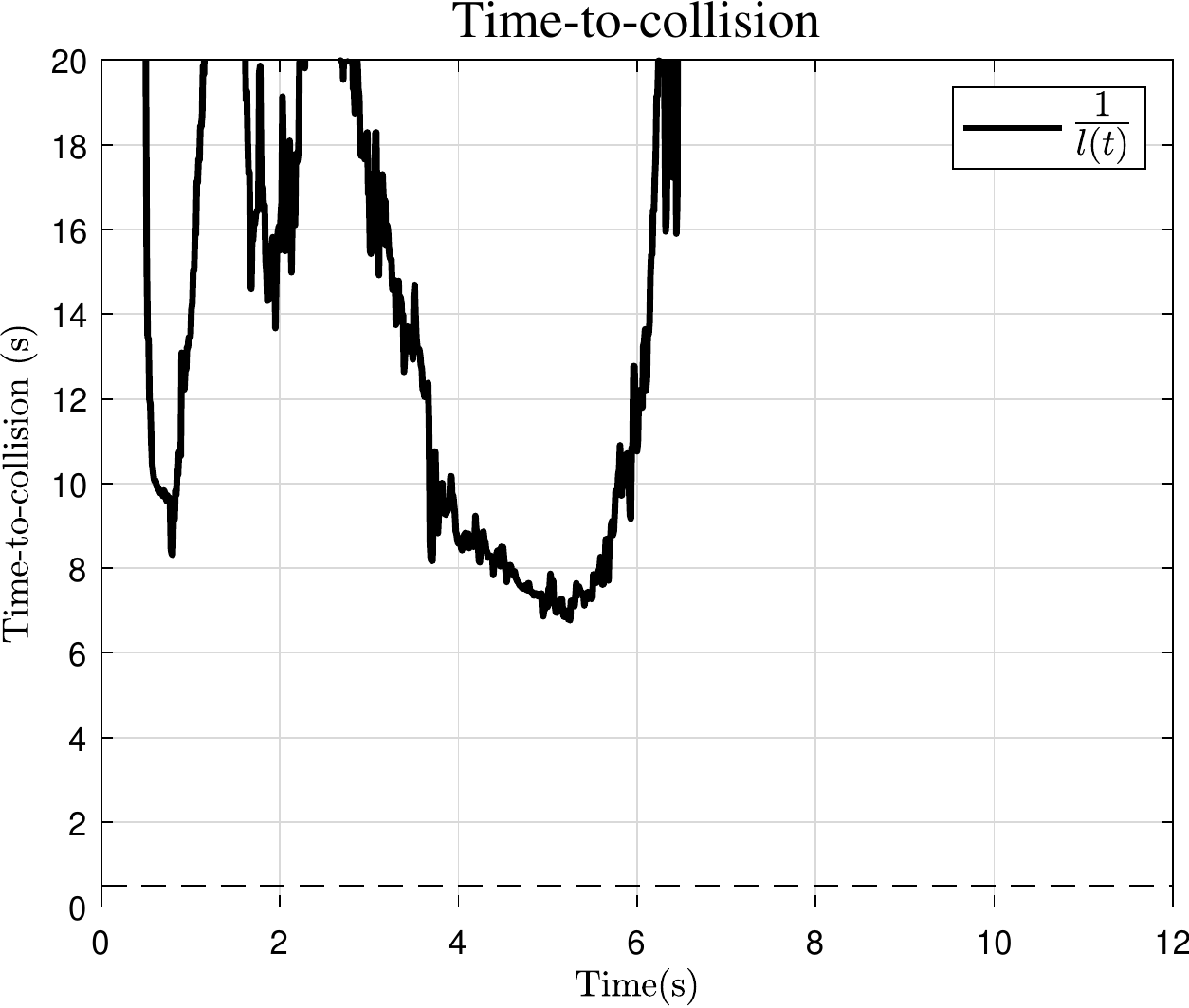}}
    \label{fig:exp1_ttc}}
	\caption{\small{Time history of signals.}} 
	\label{fig:exp1_s}
\end{figure}

\subsection{Scenario 2}

In this experiment an unknown object produces a zig-zag inducing two reactive evading maneuvers, the first of which lasts from $t=1$s until $t=4s$, and the second maneuver lasts from $t=6$ to $t=7.5$ seconds, as observed in the time-to-collision plot in Figure \ref{fig:exp4_s}. This illustrates the fact that the controller is able to engage and disengage the avoidance behavior without the knowledge of distance, overcoming the limitation of the algorithm proposed in \cite{cichella2015collision}. Figure \ref{fig:exp4} shows the successful collision avoidance, and the vehicle returning to the desired trajectory.

\begin{figure}[ht]
    \centering
    \includegraphics[width=0.6\linewidth]{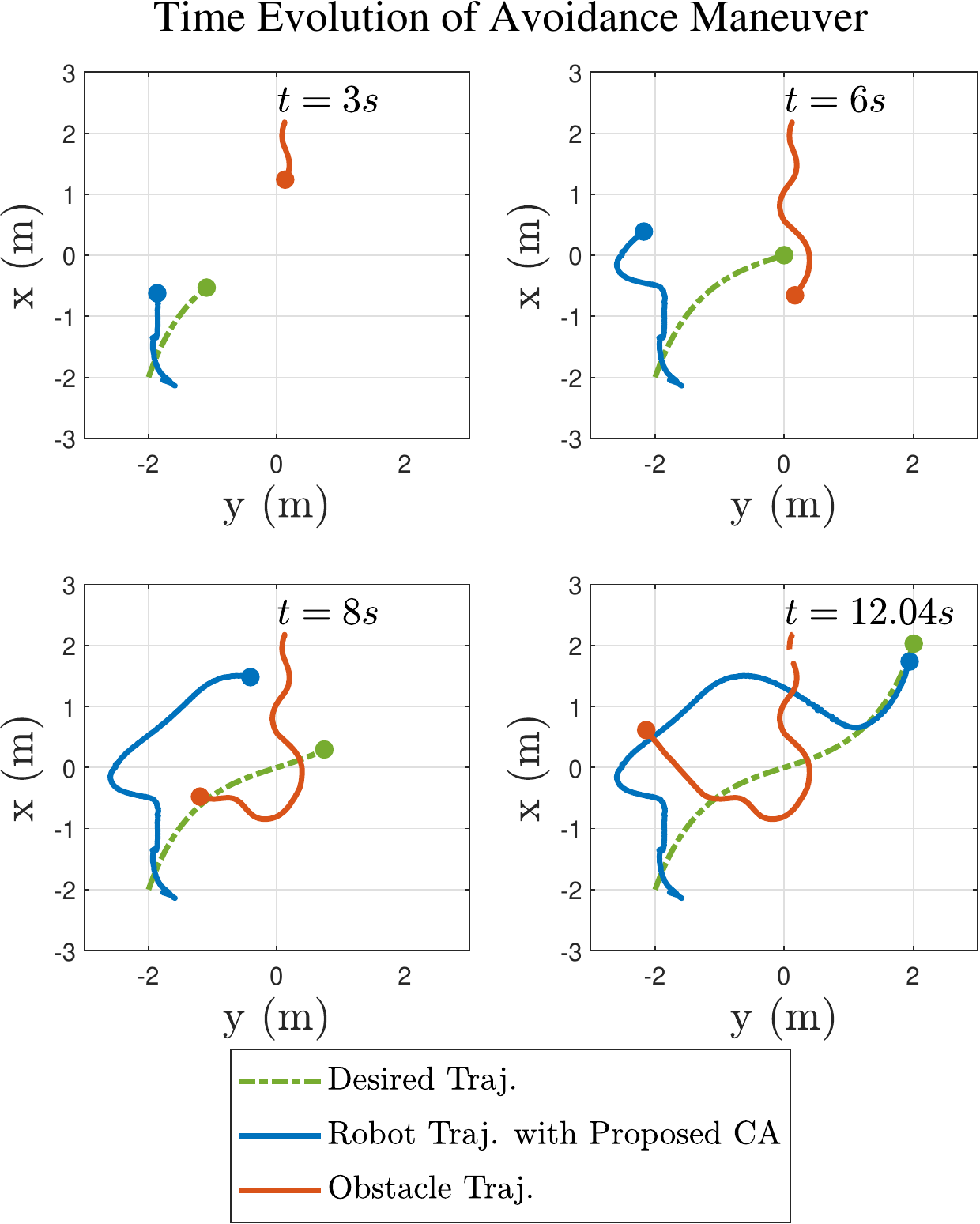}
       \caption{\small{Time evolution of robot's evading maneuver.}}
    \label{fig:exp4}
\end{figure}

\begin{figure}[ht]
	\centering
	\resizebox{.7\linewidth}{!}{
{\includegraphics[width=0.5\linewidth]{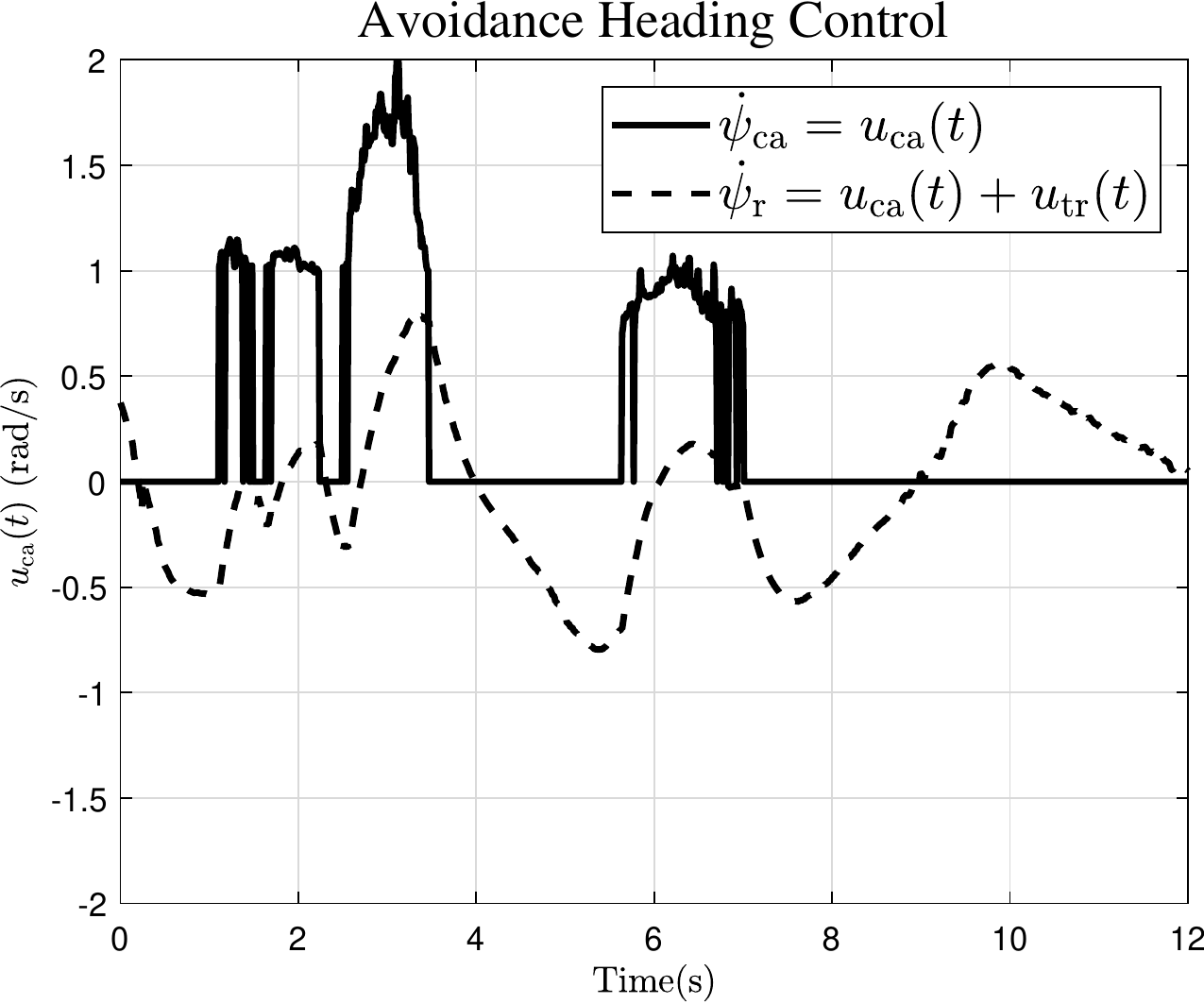}
    \label{fig:exp4_uca}}
{\includegraphics[width=0.5\linewidth]{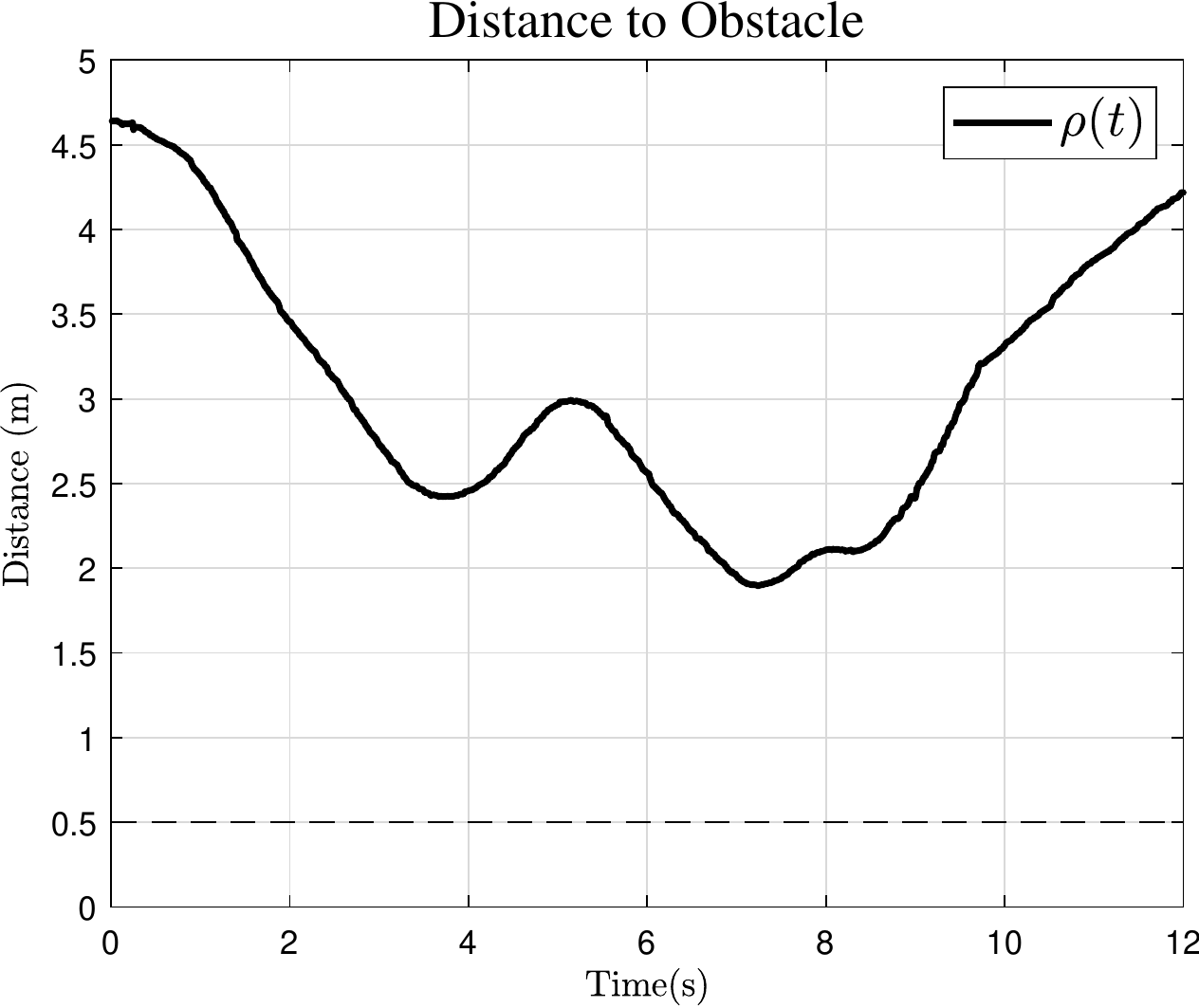}
    \label{fig:exp4_rho} }
 {\includegraphics[width=0.5\linewidth]{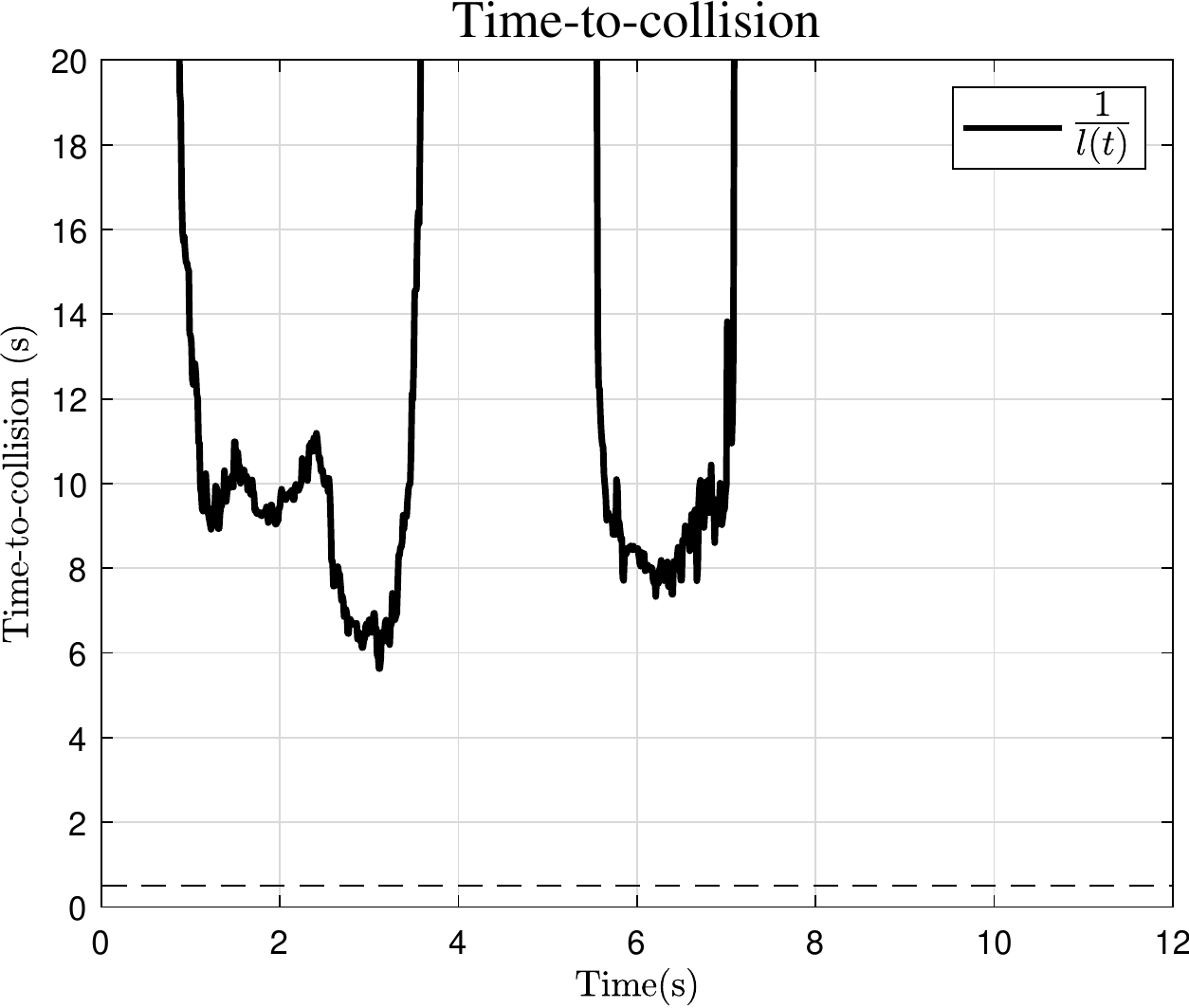} }
    \label{fig:exp4_ttc}}
	\caption{\small{Time history of signals.}} 
	\label{fig:exp4_s}
\end{figure}
\section{Conclusion}
In this paper, we addressed the issue of avoiding collision with an unknown, uncooperative pop-up obstacle with limited sensing capabilities.  The key feature of the proposed algorithms is that it does not require measurement of the distance to the obstacle. We introduced a collision avoidance algorithm that also guarantees a minimum time-to-collision. Overall, the control strategies developed in this work are designed to work along with a nominal tracking controller. Future extension of this work is to develop controllers and analysis for different vehicles, like car-like kinematics, that are  suitable for self-driving cars. In addition, the avoidance control law will be derived at the dynamics level, where bounds on control rates can be derived that are desirable in a real-world application.

\section{Acknowledgments}
This work is supported by Air Force Office of Scientific Research, NASA Langley Research Center, the National Science Foundation NRI grants \#1830639 and \#2019-04791 (project accession no. 102028 from the USDA National Institute of Food and Agriculture).
\bibliography{thesisrefs}
\bibliographystyle{ieeetr}

\end{document}